\documentclass[12pt,a4paper,fleqn]{article}

\usepackage{amsmath,amssymb,amsthm}
\usepackage{graphicx}
\usepackage{appendix}
\usepackage{tabularx}
\usepackage{multirow}

\usepackage[a4paper, left=2.5cm, right=2.5cm, top=2.8cm, bottom=2.8cm]{geometry}
\usepackage[ruled,vlined]{algorithm2e}
\usepackage{hyperref}
\hypersetup{colorlinks,
linkcolor=black,
citecolor=black,
urlcolor=blue}
\newcommand{\bI}{\text{I}}

\newcommand{\dd}{\mbox{d}}

\newtheorem{theorem}{Theorem}
\newtheorem{lemma}{Lemma}

\usepackage{setspace}

\usepackage{soul,xcolor}

\newcommand{\Var}{\mbox{Var}}

\newcommand{\wht}{\mbox{$\widehat{\theta}$}}
\newcommand{\whtl}{\mbox{$\widehat{\theta}^{(l)}$}}
\newcommand{\whel}{\mbox{$\widehat{\eta}^{(l)}$}}
\newcommand{\whtr}{\mbox{$\widehat{\theta}^{(r)}$}}
\newcommand{\swhtl}{\mbox{$\sigma_{\hat{\theta}^{(l)}}^2$}}

\makeatletter
\newcommand{\mathleft}{\@fleqntrue\@mathmargin0pt}
\newcommand{\mathcenter}{\@fleqnfalse}
\makeatother

\begin{document}
\begin{center}
{\bf \large
Deep Ensembles from a Bayesian Perspective}

\vspace{1.0cm}
Lara Hoffmann$^{1,*}$ \& Clemens Elster$^1$ \\

$^1$Physikalisch-Technische Bundesanstalt, Braunschweig and Berlin, Germany \\

$^*$e-mail: lara.hoffmann@ptb.de \\

\vspace{0.2cm}
\today
\end{center}

\begin{abstract}
Deep ensembles can be considered as the current state-of-the-art for uncertainty quantification in deep learning.
While the approach was originally proposed as a non-Bayesian technique, arguments supporting its Bayesian footing have been put forward as well.
We show that deep ensembles can be viewed as an approximate Bayesian method by specifying the corresponding assumptions.
Our findings lead to an improved approximation which results in an enlarged epistemic part of the uncertainty.
Numerical examples suggest that the improved approximation can lead to more reliable uncertainties. Analytical derivations ensure easy calculation of results.
\end{abstract}

\section{Introduction}
Deep learning has been successfully applied to numerous tasks due to its ability to learn complex relationships from data. The black box character of the employed deep neural networks, however, is a major drawback. 
Thorough testing \cite{sun2018,tian2018} and explainable AI \cite{samek2019,baldassarre2019} have been proposed to overcome this issue. Uncertainty quantification \cite{levi2019,yao2019} is another approach that can be viewed in this context. However, quantifying the uncertainty associated with the predictions of a trained network has its own merit and is particularly relevant when decisions are taken in sensitive applications such as medical diagnosis \cite{leibig2017} or autonomous driving \cite{michelmore2020}.

Several methods have been proposed to quantify the uncertainty associated with the prediction of a deep neural network \cite{abdar2020, jospin2020}.
Currently, deep ensembles \cite{lakshminarayanan2017} is one of the most common methods.
Different studies have shown that deep ensembles often outperform other methods \cite{caldeira2020, gustafsson2020, ovadia2019, scalia2020}. Furthermore, ensemble techniques scale well to high-dimensional problems and are easy to implement \cite{gustafsson2020, hoffmann2021}. These techniques can thus be seen as state-of-the-art for uncertainty quantification in deep learning.

Deep ensembles are based on a combination of multiple networks, an idea that dates back to the $20$th century. In \cite{hansen1990}, better accuracies were achieved for an ensemble of networks than when employing a single network; the use of multiple networks to account for model uncertainty has been proposed in \cite{carney1999}.  Ensemble learning has gained in popularity ever since \cite{sagi2018}.
A breakthrough in the application of ensembling for deep learning was the introduction of deep ensembles \cite{lakshminarayanan2017} in $2017$, where an ensemble of independently trained networks is treated as a uniformly weighted Gaussian mixture model.
In \cite{ashukha2020}, deep ensembles are even used to introduce a metric to measure the performance of other ensembling methods by counting the number of networks needed to reach a specified performance.

Uncertainty quantification by deep ensembles has been proposed without referring to a Bayesian background and has even been explicitly classified as a non-Bayesian approach in \cite{lakshminarayanan2017,ovadia2019,hu2020}.
Several ideas have been put forward to embed deep ensembles into a Bayesian framework. Examples comprise early stopping \cite{duvenaud2016}, regularizing parameters around different values drawn from a prior distribution \cite{pearce2020} and creating a link to Gaussian processes via the neural tangent kernel \cite{he2020}. In \cite{wilson2020}, deep ensembles are associated to Bayesian model averaging \cite{hoeting1999}, while in \cite{gustafsson2020}, they are represented as an approach that builds on an approximation of a multimodal posterior for the network parameters. 

In following this latter perspective, we argue that deep ensembles provide an approximate Bayesian inference in which the true posterior is replaced with an average of delta distributions around local maximum a posterior (MAP) estimates for the parameters of the network.
We specify the statistical model and prior distributions required to achieve this result. By using an average of delta distributions as an approximation of the posterior, the formulas applied in deep ensembles \cite{lakshminarayanan2017} follow directly.
To improve this approximation, we propose replacing the delta distributions with Gaussian distributions, which is a well-known technique from Laplace approximation \cite{ritter2018}.
A similar idea is presented in \cite{wilson2020}, where the approximated posterior distribution is generalized to a mixture of Gaussians which are independently trained as stochastic-weight-averaging Gaussians (SWAG) \cite{maddox2019}. The approximation of the posterior by a family of Gaussian distributions is well known from Bayesian neural networks \cite{lampinen2001} and variational inference in general \cite{blei2017}.

By applying variational inference, we maximize the ELBO to determine the additional parameters of the Gaussian mixture distribution. However, rather than maximizing the ELBO with respect to all parameters during network training, we propose a two-step approach. First, conventional training of the ensemble of networks yields local MAPs of the network parameters, which are taken as the locations of the Gaussian mixture distribution. In the second step, the remaining parameters are determined by maximizing the ELBO. In this way, training of the deep ensembles does not have to be altered, and our improved approximation of the posterior can be obtained via post-processing.

Previous work suggests that model uncertainty can be captured using only a few or even just one Bayesian layer at the end of the network \cite{zeng2018,brosse2020,kristiadi2020}. Therefore, we also consider only the weights of the last layer in the neural networks to be random, and do not apply any nonlinear transformations afterwards. This allows an analytical solution for the maximization of the ELBO to be derived. As a consequence, the proposed variational inference can be easily applied by simple post-processing of the results obtained by conventional deep ensembles.

The improved approximation of the posterior leads to a modification of the formulas used to perform uncertainty quantification by deep ensembles. Specifically, the epistemic part of the uncertainty is enlarged by an additional contribution. We argue that, in the context of regression problems, the epistemic part of the uncertainty can be the relevant part when the goal is to infer the regression function, and the proposed modification is then necessary to arrive at a reliable uncertainty quantification.

Our contribution is the following. We specify the statistical model, the prior distributions and the form of approximation required such that uncertainty quantification by deep ensembles \cite{lakshminarayanan2017} can be viewed as an approximate Bayesian method. This approximation is then improved by extending the family of distributions used to approximate the posterior from delta distributions to Gaussian distributions. The ELBO is maximized to determine the additional parameters, leading to a simple post-processing procedure of conventional deep ensembles. By using numerical examples, we demonstrate the impact of the proposed modified uncertainty quantification and compare its results with those obtained by conventional deep ensembles. We argue that the increased epistemic part of the uncertainty is relevant in those regression tasks where the goal is to infer the regression function.
Analytical derivations are given which allow for analytical calculations of the proposed procedure.

The paper is organized as follows. In Section 2, the conditions are specified under which uncertainty quantification by deep ensembles can be viewed as an approximate Bayesian method. The improvement of this approximation through a family of Gaussian mixture distributions is then considered in Section 3, including an analytical derivation of the additional variational parameters. In Section 4, numerical examples are presented that explore the impact of the proposed approach in the context of regression problems. Finally, conclusions are drawn in Section 5.

\section{Deep ensembles as a Bayesian approximation}
In this work we focus on regression problems. The assumed heteroscedastic regression model is given as
\begin{equation}\label{eq1.0}
y|x \sim N \left( \eta_\theta(x), \sigma_\theta^2(x) \bI \right) \,,  x\in\mathbb{R}^{p_x},\, y \in\mathbb{R}^{p_y}\,,
\end{equation}
where the conditional distribution for $y$ given $x$ is taken as a Gaussian distribution with mean $\eta_\theta (x)$ and covariance matrix
$\sigma_\theta^2(x)\bI$. Both $\eta_\theta (x)$ and $\sigma_\theta^2(x)$ are modeled by deep neural networks. The mean $\eta_\theta (x)$, viewed as a function of $x$, will be termed the regression function, and we are particularly interested in its inference. The available training data $(x_i,y_i), i=1, \ldots, N$, shall follow model \eqref{eq1.0} and are denoted by $D$.

In a Bayesian framework \cite{gelman2013,robert2007,ohagan2004}, the posterior distribution $\pi(\theta|D)$ is determined (or approximated) and used for inference. For example, as an estimate of the regression function $\eta_\theta (x)$ and its associated uncertainty one could take the posterior mean and posterior standard deviation of $\eta_\theta (x)$. If one is interested in predicting a future observation $y$ at a specified $x$, on the other hand, the posterior predictive distribution $\pi(y|x,D)$ should be considered. The posterior predictive distribution accounts for both the uncertainty in the parameters $\theta$ and the additional uncertainty of $y$ expressed by the covariance matrix $\sigma_\theta^2(x)\bI$.

The total uncertainty associated with the posterior predictive distribution of the neural network predictions can be divided into an aleatoric part and an epistemic part \cite{gal2016thesis,hullermeier2021}. While the former is related to the irreducible part of the uncertainty, the epistemic uncertainty characterizes the uncertainty about the model (i.e. the model parameters $\theta$ in our setting). The epistemic uncertainty can be reduced by increasing the size of the available training data. The aleatoric uncertainty is represented in our model \eqref{eq1.0} by the covariance matrix $\sigma_\theta^2(x)\bI$ and characterizes the uncertainty about the (future) realization of an observation $y$ given its mean $\eta_\theta (x)$.

In the deep ensembles method \cite{lakshminarayanan2017}, an ensemble of individually trained networks is considered. To predict an observation $y$, the mean of the  predictions of the single networks is formed. The (squared) uncertainty associated with this prediction is taken as the sum of the covariance matrix of the individual predictions and the estimated aleatoric part of the prediction (i.e. $\sigma_{\widehat{\theta}}^2(x)\bI$ in our model \eqref{eq1.0}).


Deep ensembles were originally introduced as a non-Bayesian method \cite{lakshminarayanan2017}. However, it has been repeatedly pointed out that deep ensembles can be viewed from a Bayesian perspective \cite{abdar2020,gustafsson2020,wilson2020}.
More precisely, conventional deep ensembles, as defined in \cite{lakshminarayanan2017}, can be viewed as an approximate Bayesian inference where the posterior $\pi(\theta|D)$ of the network parameters $\theta$ given the data $D$ is approximated by a family of delta distributions, i.e.
\begin{equation}\label{eq1.1}
\pi(\theta|D) \approx q(\theta) := \frac{1}{L} \sum_{l=1}^L \delta(\theta - \whtl)\, ,
\end{equation}
while choosing the prior $\pi (\theta)$ for the network parameters normally distributed as
\begin{equation}\label{eq0.0}
\theta \sim N \left(0, \lambda^{-1} \bI \right)\, ,
\end{equation}
where $\lambda$ is the $L2$ regularization parameter for the network parameters.

To show this, let the data $D=\{(x_i,y_i), i= 1, \ldots, N \}$ be iid samples following the regression model in \eqref{eq1.0}.
Each neural network of the ensemble is independently trained according to \cite{lakshminarayanan2017} by minimizing the loss function
\begin{equation}
loss = -L(\theta;D) + \frac{1}{2} \lambda \| \theta \|^2 \,,\
L(\theta;D) = - \frac{1}{2} \sum_{i=1}^N \left(
\frac{ \| y_i-\eta_\theta(x_i) \|^2}{\sigma_\theta^2(x_i)} + p_y \log(\sigma_\theta^2(x_i)) \right)\, ,
\end{equation}
where $\Vert . \Vert$ is the $L^2$ norm, and $L(\theta;D)$ equals (up to a constant) the log likelihood for the statistical model in \eqref{eq1.0}.
The parameter estimates $\whtl, l=1, \ldots, L$, obtained for the ensemble of $L$ networks, are at the same time (local) estimates for the maximum a posteriori probability (MAP) of $\theta$ when assigning the prior \eqref{eq0.0}.
By taking the average of $L$ probability mass functions at the local MAPs as an approximation of the posterior $\pi(\theta|D)$
(cf. \eqref{eq1.1}), 
the approximation of the posterior predictive distribution for the statistical model in \eqref{eq1.0} is then given by
\begin{equation}\label{eq1.1b}
\pi(y|x,D) = \int \pi(\theta|D) N \left( y; \eta_\theta(x),\sigma_\theta^2(x) \bI \right) \dd \theta 
 \approx \frac{1}{L} \sum_{l=1}^L N \left( y; \eta_{\widehat{\theta}^{(l)}}(x),\sigma_{\widehat{\theta}^{(l)}}^2(x) \bI \right) \,,
\end{equation}
which equals the expression given in \cite{lakshminarayanan2017} for the deep ensembles.

The average over delta distributions taken at point estimates \eqref{eq1.1} is a crude approximation of the posterior.
Nonetheless, deep ensembles outperform many other Bayesian approaches due to their ability to explore different modes of the posterior distribution \cite{fort2019}. Figure \ref{fig solution space}, which is inspired by \cite{fort2019} (Fig. 1) and \cite{wilson2020} (Fig. 3), illustrates this behavior.
\begin{figure}[t]
    \centering
    \includegraphics[scale = 0.6]{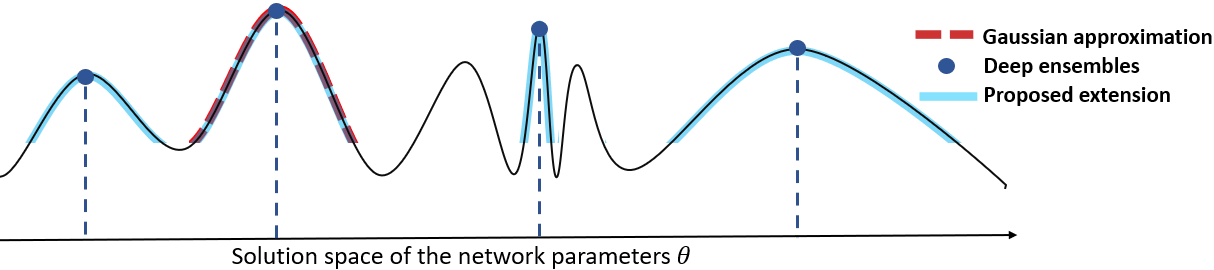}
    \caption{This example figure illustrates the solution space of the network parameters on the $x$-axis and the negative loss function on the $y$-axis. The deep ensembles explore multiple modes in comparison to other Bayesian approaches which focus on a single mode but consider multiple moments. The proposed extension combines these two ideas.}\label{fig solution space}
\end{figure}

By using the approximation \eqref{eq1.1b}, mean and covariance matrix of the posterior predictive distribution are obtained as
\mathleft
\begin{eqnarray}\label{eq1.2}
E \left( y |x,D \right) &=& \frac{1}{L} \sum_{l=1}^L  \eta_{\widehat{\theta}^{(l)}}(x) \,, \\
Cov \left( y |x,D \right) &=&
\frac{1}{L} \sum_{l=1}^L  \left\{ \left(\eta_{\widehat{\theta}^{(l)}}(x)-E \left( y |x,D \right)\right)
\left(\eta_{\widehat{\theta}^{(l)}}(x)-E \left( y |x,D \right)\right)^T + \sigma_{\widehat{\theta}^{(l)}}^2(x) \bI \right\} \,,\label{eq1.3}
\end{eqnarray}
\mathcenter
where $E$ stands for expectation and $Cov$ for covariance, cf. Appendix \ref{appendix a1}. Note that \eqref{eq1.2} yields the inferred regression function $\eta_\theta(x) = E \left( y |x,D \right)$ according to the statistical model \eqref{eq1.0}, while \eqref{eq1.3} describes the total uncertainty associated with the prediction of the network ensemble. Mean and covariance matrix of the posterior predictive distribution are in accordance with the results given in \cite{lakshminarayanan2017} for the original approach of the deep ensembles. 

In \cite{lakshminarayanan2017}, the comparison between different methods is based on the posterior predictive distribution. The covariance matrix \eqref{eq1.3} comprises the reducible epistemic part of the uncertainty as well as the irreducible aleatoric part. However, when the goal is to infer the underlying regression function, an accurate estimation of the epistemic uncertainty is essential. The epistemic uncertainty is given by the covariance matrix of the single estimates of the regression function obtained by the ensemble of trained networks.
Mean and covariance matrix of the regression function are deduced in the following lemma.

\begin{lemma}\label{T2}
The approximate Bayesian inference associated with deep ensembles yields a posterior $\pi(\eta|x,D)$ for the regression function $\eta\equiv\eta_\theta (x)$ given by 
\begin{equation}
\pi(\eta|x,D) = \frac{1}{L} \sum_{l=1}^L \delta(\eta- \eta_{\widehat{\theta}^{(l)}}(x)) \,,
\end{equation}
with mean and covariance matrix given by
\begin{eqnarray}
E \left( \eta | x,D \right)  &=& \frac{1}{L} \sum_{l=1}^L \eta_{\widehat{\theta}^{(l)}}(x)\, , \label{eq2.1} \\ \label{eq2.2}
Cov \left( \eta |,x,D \right) &=&
\frac{1}{L} \sum_{l=1}^L  \left\{ \left(\eta_{\widehat{\theta}^{(l)}}(x)-E \left( \eta | x,D \right) \right)
\left(\eta_{\widehat{\theta}^{(l)}}(x)-E \left( \eta | x,D \right) \right)^T \right\} \, .
\end{eqnarray}
\end{lemma}
\begin{proof}
The proof is given in Appendix \ref{appendix T2}.
\end{proof}

Note that the expectation of the posterior for the regression function \eqref{eq2.1} equals the expectation of the posterior predictive distribution \eqref{eq1.2}, while its covariance matrix \eqref{eq2.2} contains only the epistemic part of the covariance matrix of the posterior predictive distribution \eqref{eq1.3}.

\section{Extension to Gaussian mixture distributions}
If the aleatoric part of the covariance matrix of the posterior predictive distribution \eqref{eq1.3} is dominant, the covariance matrix of the regression function \eqref{eq2.2} is expected to be significantly smaller than that of the posterior predictive distribution (cf. example in Figure \ref{ex1} and results in Table \ref{table: results}). It follows that the (crude) approximation of the posterior in \eqref{eq1.1} with a family of delta distributions, which ignores the nonzero width of the true posterior around its (local) MAPs, is irrelevant for the total uncertainty in these cases. Nevertheless, this approximation can lead to a significant underrating of the epistemic uncertainty.

An improved approximation of the posterior distribution $\pi(\theta|D)$ is given by 
\begin{equation}\label{eq1.11}
q(\theta) = \frac{1}{L} \sum_{l=1}^L N(\theta;\whtl, \gamma_l \bI) \,,
\end{equation}
where the $\gamma_l$ correspond to the variances around the (local) MAPs $\whtl$.
The approximation \eqref{eq1.11} of the posterior extends the approximation in \eqref{eq1.1} by allowing for a finite width of the posterior around the (local) MAPs, and it reduces to \eqref{eq1.1} for $\gamma_l \to 0$.

The general idea behind variational inference \cite{blei2017} is to approximate the posterior $\pi(\theta|D)$ by a distribution that can be more easily handled. Specifically, one seeks the best approximation of the posterior within a parametric family of distributions such as the Gaussian mixture distribution in \eqref{eq1.11}. One way to do so is to use a Laplace approximation, for which the parameters of a single Gaussian approximation are calculated from the local properties of the posterior at the MAP (cf. \cite{ritter2018}). In deep learning approaches, however, the parameters are typically determined by minimizing the Kullback-Leibler divergence $KL(q(\theta)\Vert \pi(\theta|D))$ between the chosen family of distributions and the posterior. This is equivalent to maximizing the evidence lower bound (ELBO)
\begin{equation}\label{eq3.1}
ELBO = E_q \left[\log p(D|\theta)\right] - KL \left( q(\theta) \| \pi(\theta) \right) \,.
\end{equation}

Instead of maximizing the ELBO with respect to all parameters during network training, we here suggest a two-step procedure. First, the networks are independently trained in accordance with the conventional deep ensembles training procedure. The resulting (local) MAPs of the network parameters are taken as the first moments of the Gaussian mixture distributions. Second, the remaining parameters of \eqref{eq1.11}, which are the variances $\gamma$, are deduced by maximizing the ELBO \eqref{eq3.1} in a post-processing step. This can be done 
numerically for example via a grid search.

In the following,we assume that the only random weights in the network are those in the last layer.
Furthermore, we do not consider the randomness of $\sigma^2_\theta(x)$, but only that of $\eta_\theta(x)$.
The reason is that we are interested in the epistemic part of the uncertainty associated with the regression function. Also, $\sigma^2_\theta$ would not depend on the last layer but rather be an individual network parameter, when choosing a homoscedastic approach instead,
which could be a promising alternative to the heteroscedastic model \eqref{eq1.0}. Maximizing the ELBO can then be performed analytically, as stated in Theorem \ref{T3} below (cf. also Figure \ref{net} for visualizing the parameters treated as random in the proposed approach). Mean and covariance matrix of the regression function for the improved approximate Bayesian inference are given in Theorem \ref{T4} below, while Algorithm \ref{algo1} summarizes the proposed approach.

\begin{figure}[h]
    \centering
    \includegraphics[scale = 0.8]{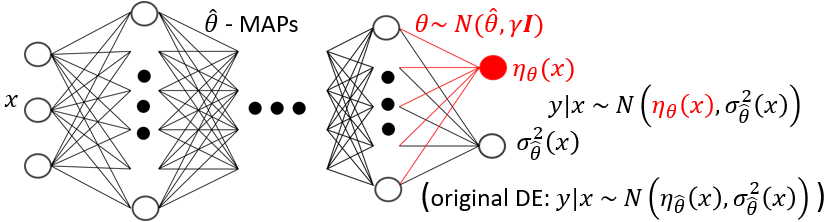}
    \caption{The proposed model of a single network is visualized here. A sample $y|x$ from the posterior predictive distribution is drawn from the network outputs $N(\eta_\theta (x),\sigma_{\hat{\theta}}^2 (x))$. All network weights are fixed to the MAP estimates after training, except the last linear unit of the predicted mean value $\eta_\theta (x)$. The weights of this last linear unit are normally distributed around the trained MAPs $\wht$ with variances $\gamma_l, l=1, \ldots, L$.}\label{net}
\end{figure}

\begin{theorem}\label{T3}
Let only the weights of the last layer in the neural network, which are related to the prediction of the regression function $\eta_\theta (x)$, be random and assume a linear transformation from the last layer to the network output. Approximate maximization of the ELBO \eqref{eq3.1} can then be done analytically when considering the approximation of the posterior distribution given in \eqref{eq1.11},
leading to the expression \eqref{eq7.9} in the Appendix. Furthermore, the variances in \eqref{eq1.11} tend to zero for an infinite amount of training data.
\end{theorem}
\begin{proof}
The proof is given in Appendix \ref{appendix T3}.
\end{proof}
\begin{algorithm}[h]
\SetAlgoLined
\KwResult{$q(\theta) := \frac{1}{L} \sum_{l=1}^L N(\theta;\whtl, \gamma_l \bI)$, instead of $q(\theta) := \frac{1}{L} \sum_{l=1}^L \delta(\theta - \whtl)$\;}
 1. initialize and independently train networks $\eta_{\widehat{\theta}^{(l)}}$, $l=1,\ldots,L,$ in accordance with the deep ensembles approach\;
 2. compute $\gamma_l$, $l=1,\ldots,L,$ according to \eqref{eq7.9}.
 \caption{improved approximation of} the posterior $\pi (\theta|D)$\label{algo1}
\end{algorithm}

\begin{theorem}\label{T4}
The improved approximation for the Bayesian inference associated with deep ensembles (from delta to normal distributions, cf. equation \eqref{eq1.11}) yields a posterior $\pi(\eta|x,D)$ for the
regression function $\eta\equiv\eta_\theta (x)$ given by 
\begin{equation}\label{eq1.18}
\pi(\eta|x,D) = \frac{1}{L} \sum_{l=1}^L \int \delta(\eta- \eta_\theta(x)) N \left( \theta; \whtl, \gamma_l \bI \right) \dd \theta \,,
\end{equation}
with mean and covariance matrix given approximately by
\mathleft
\begin{eqnarray}
E \left( \eta | x,D \right)  &\approx &  \frac{1}{L} \sum_{l=1}^L  \eta_{\widehat{\theta}^{(l)}}(x) \,, \label{eq6.1}\\\label{eq6.2}
Cov \left( \eta |x,D \right) 
&\approx& 
\frac{1}{L} \sum_{l=1}^L \left\{ \left(\eta_{\widehat{\theta}^{(l)}}(x) - E \left( \eta | x,D \right) \right)
\left(\eta_{\widehat{\theta}^{(l)}}(x) - E \left( \eta | x,D \right) \right)^T + \gamma_l J_l J_l^T \right\},
\end{eqnarray}
\mathcenter
where $J_l= \partial \eta / \partial \theta$ is evaluated at $\theta=\whtl$.
The Jacobian matrix $J_l$ can be easily calculated in terms of the output of the next-to-last layer when considering only the weights of the last layer in the neural networks to be random, and when applying no nonlinear transformation to the output of the last layer.
\end{theorem}
\begin{proof}
The proof is given in Appendix \ref{appendix T4}.
\end{proof}

The improved approximation of the Bayesian inference associated with deep ensembles yields the same expectation \eqref{eq6.1} for the regression function as the original approximation \eqref{eq2.1}, together with a covariance matrix \eqref{eq6.2} that includes the additional term $\gamma_l J_l J_l^T$ (compared to \eqref{eq2.2}). In this way, the epistemic part of the uncertainty of the conventional deep ensembes is enlarged.

Samples from the posterior \eqref{eq1.18} for the regression function can be easily obtained by first sampling an $l \in \{1, \ldots,L\}$ with probability $\frac{1}{L}$. Second, a $\theta^{(l)}$  is sampled from $N(\widehat{\theta}^{(l)},\gamma_l \bI)$ and third, $\eta_{\theta^{(l)}}(x)$ is evaluated. To sample from the posterior predictive distribution, a fourth step is added: sample a $y$ from $N \left( \eta_\theta(x), \sigma_{\hat{\theta}^{(l)}}^2(x) \bI \right)$ of the $l$th network. The sampling procedure is summarized in Algorithm \ref{algo2}.

\begin{algorithm}[h]
\SetAlgoLined
\KwResult{sample from $\pi(\eta |x,D)$: steps $1$-$3$; sample from $\pi(y |x,D)$: steps $1$-$4$\;}
 1. sample $l$th network uniformly from $l\in\{1, \ldots,L\}$\;
 2. sample network weights $\theta^{(l)}$ from $N(\widehat{\theta}^{(l)},\gamma_l \bI)$\;
 3. evaluate $\eta_{\theta^{(l)}}(x)$\;
 4. evaluate $\sigma_{\hat{\theta}^{(l)}}^2(x)$ and sample from $N \left( \eta_{\theta^{(l)}}(x), \sigma_{\hat{\theta}^{(l)}}^2(x) \bI \right)$.
 \caption{sample from the posterior of the regression function $\pi(\eta |x,D)$ and from the posterior predictive distribution $\pi(y |x,D)$}\label{algo2}
\end{algorithm}

\section{Results}
In this section, we numerically compare the classical deep ensembles \cite{lakshminarayanan2017} (DE) to the introduced extended approach (DE extended).
The implemented code, together with examples, is provided\footnote{\href{https://gitlab1.ptb.de/hoffma31/bayesiandeepensembles}{BayesianDeepEnsembles}}.
Following Algorithm \ref{algo1}, the network training procedure is exactly the same for the classical and the extended deep ensembles. The inferred regression function is identical in both cases, namely the average over the estimates of all ensemble members $\frac{1}{L}\sum_{l=1}^L\eta_{\hat{\theta}^{(l)}} (x)$ at the trained MAPs $\wht$. The only difference is the additional term  of the extended covariance matrix in \eqref{eq6.2}, which is related to the epistemic part of the uncertainty.

The architecture of all networks in this work is fully connected, with three hidden layers consisting of $128,\ 64$, and $32$ neurons, respectively. Uncertainties are understood as $95\%$ credible intervals calculated in terms of an assumed Gaussian distribution using the approximate variance formulas (i.e. $1.96$ times the standard deviation). Coverage probabilities refer to the number of test cases that lie within the $95\%$ symmetric credible interval. Training and test data are always disjoint and all obtained results are shown in respect to the test set.

\subsection*{Experiments on simulated data}
Training is performed over $60$ epochs with a learning rate drop factor of $0.1$ for the last five epochs and a mini batch size of $64$. The regularization parameter $\lambda$ is always set to one over the number of available training data.
All data of the output space are normalized by subtracting the mean and dividing the difference by the standard deviation of the corresponding training set. The input data are uniformly drawn from the support $[-1,1]$. The training and test sets are disjoint. All presented results refer to the test data. The hyperparameters and the network architecture are not fine-tuned, as the achieved accuracy suffices for the purpose of this work.

\begin{figure}[h]
    \centering
    \includegraphics[scale = 1.025]{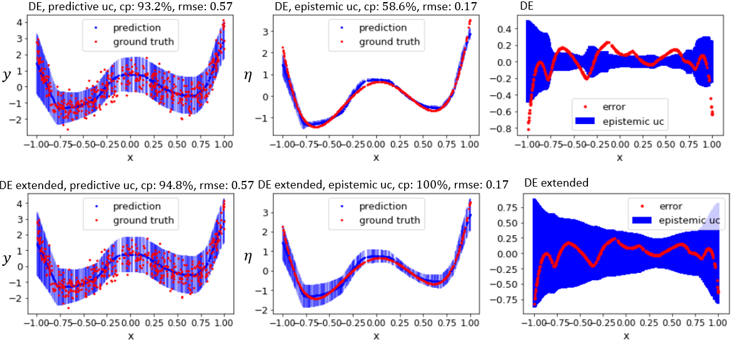}
    \caption{The results are plotted for the classical deep ensembles (first row) and the proposed extension (second row). The coverage probabilities (cp) and root-mean-squared errors (rmse) are given for the estimated posterior predictive distribution (first column) and the regression function (second column). The estimates, predictions and corresponding uncertainties are plotted (in blue) together with the underlying disturbed and undisturbed ground truth data (in red). In the third column, the epistemic uncertainties are plotted together with the errors of the estimated regression function to the ground truth.}\label{ex1}
\end{figure}

\begin{figure}[h]
    \centering
    \includegraphics[scale = 0.9]{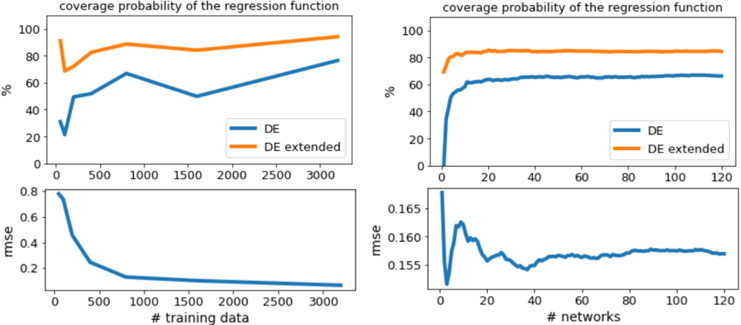}
    \caption{The coverage probabilities for the regression function of the classical deep ensembles approach and its suggested extension are plotted in the first row in dependence on an increasing amount of training data and an increasing number of ensemble members, respectively. The second row shows the corresponding root-mean-squared errors (which are the same for both approaches)}.\label{ex2}
\end{figure}

Figure \ref{ex1} shows the predicted regression function and the resulting total and epistemic uncertainties for the classical and extended deep ensembles together with the ground truth. The data $y(x)\in\mathbb{R},\ x\in [-1,1],$ are normally distributed around the regression function $\eta (x) = \frac{1}{2} \left((4.5x)^4-(18x)^2+22.5x\right)$ with a standard deviation of $10$. An ensemble of $10$ networks was trained using $200$ data samples with a learning rate equal to $1/200$. The total uncertainty is much greater than the epistemic part of the uncertainty for the classical deep ensembles, when comparing the first two images of the first row. This means that the aleatoric part of the uncertainty is dominant. The total uncertainty almost perfectly covers the underlying disturbed data and hence, the approximation of the posterior predictive distribution seems to work well for both approaches. However, the second column shows that the uncertainty for the inferred regression function covers the ground truth in less than $60\%$ of the considered test cases for the conventional deep ensembles, i.e. the conventional deep ensembles underrate the uncertainty.
In contrast, the extended deep ensembles cover the regression function completely, as shown in the second row (second and third image), and hence yield a more reliable, yet to some extent conservative, uncertainty quantification.

Another comparison between the classical and extended deep ensembles is given in Figure \ref{ex2}, which analyzes the dependence of the amount of training data and the number of trained ensemble members on the epistemic uncertainty. The data samples $y(x)\in\mathbb{R},\ x\in [-1,1]^2,$ are normally distributed around the regression function $\eta (x) = \sum_{i=1}^2 (1.5x_i-1)^2 (1.3x_i+1)^2$ with a standard deviation of $0.2$.
The results of the first column of Figure \ref{ex2} are based on an ensemble of $10$ networks which are trained with a learning rate of $0.001$. The ensembles yielding the plots in the second column are trained using $600$ data samples with a learning rate of $1/600$. The results show that the extended approach consistently outperforms the classical deep ensembles with a coverage probability of over $80\%$ in almost every setting. Even for a single network, the coverage of the extended method is about $70\%$, while there is no epistemic uncertainty at all in the classical method for this case.
We observe in this example that the uncertainty is not further improved when the number of networks for the ensemble is increased.
The uncertainties of the two approaches are expected to converge for an infinite amount of training data, which seems to be supported by the observed results. Finally, we note that the root-mean-squared error of the predicted regression function is the same for the two approaches since the same estimates are reached. It decreases for an increasing amount of training data and is more stable for an increasing number of ensemble members.

\subsection*{Experiments on real data}
The proposed extension is tested on four real data sets downloaded from the UCI data base \cite{dua2019}. A summary of the data sets (e.g. amount of training and test data) and the chosen hyper parameters for the network training are given in Table \ref{table: parameter}. All data are normalized, i.e. the mean is subtracted and the difference is divided by its standard deviation. All networks are trained using the Adam optimizer \cite{kingma2014adam} and a constant learning rate up to the last five epochs where the learning rate drops by a factor of $0.5$ in every epoch. 

\begin{table}[h]
\begin{tabularx}{\textwidth}{ l||X|l|X|l } 
\hline
 \multicolumn{5}{l}{\textbf{Data information}} \\ \hline
 name & \begin{tabular}[c]{@{}l@{}}QSAR aquatic\\ toxicity\end{tabular} & \begin{tabular}[c]{@{}l@{}}Yacht\\ hydrodynamics\end{tabular} & \begin{tabular}[c]{@{}l@{}}Blog\\ feedback\end{tabular} & \begin{tabular}[c]{@{}l@{}}Year prediction \\ MSD\end{tabular}  \\ \hline
 reference & \cite{dua2019,cassotti2014prediction} & \cite{dua2019,gerritsma1981geometry,ortigosa2007neural} & \cite{dua2019,buza2014feedback} & \cite{dua2019,bertin-mahieux2011} \\ \hline
 \# training data & $437$ & $246$ & $52,397$ & $463,715$ \\ \hline
 \# test data & $109$ & $62$ & $7,624$ & $51,630$ \\ \hline
 input dimension & $8$ & $6$ & $276$ & $90$ \\ \hline
 \multicolumn{5}{l}{} \\
 \multicolumn{5}{l}{\textbf{Hyper parameters for the network training}} \\ \hline
 learning rate & $1e^{-3}$ & $1e^{-3}$ & $1e^{-5}$ & $5e^{-5}$ \\ \hline
 $\lambda$ & $1e^{-2}$ & $1e^{-3}$ & $1e^{-5}$ & $1e^{-4}$ \\ \hline 
 mini batch size & $128$ & $8$ & $64$ & $128$ \\ \hline
 epochs & $60$ & $100$ & $50$ & $40$ \\ \hline
\end{tabularx}
\caption{This table provides general information about the considered data sets and specifies the hyperparameters used for the network training. $\lambda$ is the $L2$ regularization parameter for the network parameters (cf. \eqref{eq0.0}).
For all four examples the output dimension equals one.}\label{table: parameter}
\end{table}

The results are shown in Table \ref{table: results} for the classical deep ensembles \cite{lakshminarayanan2017} and the proposed extension. All investigations are conducted for an ensemble of five neural networks and an ensemble of ten neural networks, respectively. Note that the root-mean-squared error (RMSE) of the predicted regression function based on the given (noisy) target values and the aleatoric uncertainty are the same for both methods since the proposed extension has only an effect on the epistemic uncertainty via the post-processing step. The epistemic and total coverages describe the relative amount of target data that lie in the $95\%$ symmetric credible intervals. The epistemic coverage relates only to the epistemic uncertainty, while the total coverage relates to both the epistemic and the aleatoric uncertainty, i.e. the total uncertainty. A ``good" uncertainty estimation would have a total coverage of $95\%$, which is roughly true for both variants and the here considered examples. The epistemic uncertainty, however, relates to the distance of the inferred regression function to the underlying ground truth, which is unknown. The epistemic coverage therefore refers to the same noisy target values as the total coverage, and is expected to be smaller than $95\%$.

\begin{table}[h]
\begin{tabular}{|ll||ll|ll|}
\hline
\multicolumn{2}{|l||}{method} & \multicolumn{2}{l|}{deep ensembles} & \multicolumn{2}{l|}{proposed extension} \\ \hline

\multicolumn{2}{|l||}{\# ensemble members} & \multicolumn{1}{l|}{$5$} & $10$ & \multicolumn{1}{l|}{$5$} & $10$ \\ \hline \hline

\multicolumn{1}{|l|}{\multirow{4}{*}{\begin{tabular}[c]{@{}l@{}}QSAR\\ aquatic\\ toxicity\end{tabular}}} & RMSE & \multicolumn{1}{l|}{$0.555$} & $0.557$ & \multicolumn{1}{l|}{$0.555$} & $0.557$ \\ \cline{2-6}

\multicolumn{1}{|l|}{} & epistemic coverage in $\%$ & \multicolumn{1}{l|}{$24.8$} & $28.4$ & \multicolumn{1}{l|}{$42.2$} & $47.7$ \\ \cline{2-6} 

\multicolumn{1}{|l|}{} & total coverage in $\%$ & \multicolumn{1}{l|}{$95.4$} & $95.4$ & \multicolumn{1}{l|}{$95.4$} & $95.4$ \\ \cline{2-6} 

\multicolumn{1}{|l|}{} & ratio epistemic-aleatoric variance & \multicolumn{1}{l|}{$0.030$} & $0.036$ & \multicolumn{1}{l|}{$0.088$} & $0.098$ \\ \hline \hline 


\multicolumn{1}{|l|}{\multirow{4}{*}{\begin{tabular}[c]{@{}l@{}}Yacht\\ hydrodynamics\end{tabular}}} & RMSE & \multicolumn{1}{l|}{$0.085$} & $0.078$ & \multicolumn{1}{l|}{$0.085$} & $0.078$ \\ \cline{2-6}

\multicolumn{1}{|l|}{} & epistemic coverage in $\%$ & \multicolumn{1}{l|}{$75.8$} & $87.1$ & \multicolumn{1}{l|}{$82.3$} & $90.3$ \\ \cline{2-6} 

\multicolumn{1}{|l|}{} & total coverage in $\%$ & \multicolumn{1}{l|}{$98.4$} & $98.4$ & \multicolumn{1}{l|}{$98.4$} & $98.4$ \\ \cline{2-6} 

\multicolumn{1}{|l|}{} & ratio epistemic-aleatoric variance & \multicolumn{1}{l|}{$0.284$} & $0.383$ & \multicolumn{1}{l|}{$0.412$} & $0.509$ \\ \hline \hline 


\multicolumn{1}{|l|}{\multirow{4}{*}{\begin{tabular}[c]{@{}l@{}}Blog\\ feedback\end{tabular}}} & RMSE & \multicolumn{1}{l|}{$0.684$} & $0.683$ & \multicolumn{1}{l|}{$0.684$} & $0.683$ \\ \cline{2-6}

\multicolumn{1}{|l|}{} & epistemic coverage in $\%$ & \multicolumn{1}{l|}{$41.0$} & $44.7$ & \multicolumn{1}{l|}{$41.6$} & $45.0$ \\ \cline{2-6} 

\multicolumn{1}{|l|}{} & total coverage in $\%$ & \multicolumn{1}{l|}{$96.8$} & $96.8$ & \multicolumn{1}{l|}{$96.8$} & $96.8$ \\ \cline{2-6} 

\multicolumn{1}{|l|}{} & ratio epistemic-aleatoric variance & \multicolumn{1}{l|}{$0.093$} & $0.092$ & \multicolumn{1}{l|}{$0.094$} & $0.092$ \\ \hline \hline 


\multicolumn{1}{|l|}{\multirow{4}{*}{\begin{tabular}[c]{@{}l@{}}Year\\ prediction\\ MSD\end{tabular}}} & RMSE & \multicolumn{1}{l|}{$0.795$} & $0.795$ & \multicolumn{1}{l|}{$0.795$} & $0.795$ \\ \cline{2-6}

\multicolumn{1}{|l|}{} & epistemic coverage in $\%$ & \multicolumn{1}{l|}{$20.4$} & $22.8$ & \multicolumn{1}{l|}{$20.5$} & $22.8$ \\ \cline{2-6} 

\multicolumn{1}{|l|}{} & total coverage in $\%$ & \multicolumn{1}{l|}{$95.2$} & $95.3$ & \multicolumn{1}{l|}{$95.2$} & $95.3$ \\ \cline{2-6} 

\multicolumn{1}{|l|}{} & ratio epistemic-aleatoric variance & \multicolumn{1}{l|}{$0.024$} & $0.027$ & \multicolumn{1}{l|}{$0.024$} & $0.027$ \\ \hline 


\end{tabular}\caption{The results are shown for the classical deep ensembles and the proposed extension for the four data sets from Table \ref{table: parameter}. The root-mean-squared error (RMSE) of the predicted regression function and the epistemic and total coverages, i.e. the relative amount of target data that lie in the $95\%$ symmetric credible intervals, are based on the given (noisy) target values. The ratio describes the mean over the test set of the estimated epistemic variances divided by the estimated aleatoric variances.}\label{table: results}
\end{table}

The ratio between the epistemic and aleatoric variances describes the proportion of the two estimated uncertainty sources. Furthermore, this ratio shows how the epistemic part of the uncertainty changes for the proposed extension in comparison the original method. In the considered examples, the aleatoric part dominates the uncertainty, because the ratio of the epistemic and aleatoric variances is always smaller than one. However, the proposed extension has a significant impact on the estimated epistemic uncertainties in the first two data sets (QSAR aquatic toxicity and yacht hydrodynamics) where only comparatively few data samples are available. Here, the increase of the epistemic uncertainty and the epistemic coverage for the proposed extension of deep ensembles can be considered an improvement, because the total coverage remains well calibrated without being further increased. In contrast, the proposed extension shows no effect on the last data set (year prediction MSD), which means that the proposed approximation of the posterior of the network parameters \eqref{eq1.11} through a family of normal distributions can safely be replaced with a family of delta distributions according to the classical deep ensembles \eqref{eq1.1}. 
This finding is also in accordance with the expectation that widths of the normal distributions tend to zero when the amount of training data become large (cf. Theorem \ref{T2}).

\section{Conclusions and Outlook}
In this work, we have explicitly derived the deep ensembles \cite{lakshminarayanan2017} as an approximation of a conventional Bayesian inference by stating the underlying approximations and assumptions, including the required statistical model and choice of prior. We have developed an extension to the deep ensembles that improves its uncertainty quantification without losing any of the advantages of the approach, and without altering its training procedure or inferred regression function. This is achieved by replacing the delta distributions that serve as an approximation of the posterior for the network parameters in a Bayesian framework with normal distributions that have finite variances. The corresponding variances can be computed analytically in a post-processing step by minimizing the Kullback-Leibler divergence between the approximate posterior and the true posterior.
The introduced method can also be employed to easily obtain an epistemic uncertainty for a single trained neural network.

The proposed extension of the deep ensembles yields an additional term to the covariance matrix of the estimated regression function that results in an enlarged epistemic uncertainty.
Simple numerical experiments illustrate the effect of the extended approach, leading to a more appropriate uncertainty quantification for the inferred regression function than classical deep ensembles.

The aim of this paper was to explicitly state the underlying approximations and assumptions needed to view the deep ensembles as an approximate Bayesian method, and to introduce a straightforward and easy-to-implement extension that leads to an improved uncertainty quantification for this state-of-the-art method. The basis of the proposed extension is similar to ideas that have already been put forward in related work \cite{maddox2019,wilson2020}. 
In contrast to those approaches, however, the extension of deep ensembles presented here does not alter the actual method, but simply adds a post-processing step. 
Future work could address the adaption to classification problems, replacing the linear unit of the network output with a convolution layer for imaging tasks, and relaxing the uniform weights $1/L$ of the Gaussian mixture distributions. Combining our extension with other methods that improve the accuracy of deep ensembles \cite{wenzel2020} or their extrapolation power \cite{pearce2020} could be analyzed as well.

\section*{Acknowledgements}
The authors thank Shinichi Nakajima for proofreading the manuscript and valuable remarks.

\bibliographystyle{ieeetr}
\bibliography{bibli}

\begin{thebibliography}{10}

\bibitem{sun2018}
Y.~Sun, X.~Huang, D.~Kroening, J.~Sharp, M.~Hill, and R.~Ashmore, ``Testing
  deep neural networks,'' {\em arXiv preprint arXiv:1803.04792}, 2018.

\bibitem{tian2018}
Y.~Tian, K.~Pei, S.~Jana, and B.~Ray, ``Deeptest: Automated testing of
  deep-neural-network-driven autonomous cars,'' in {\em Proceedings of the 40th
  {I}nternational {C}onference on {S}oftware {E}ngineering}, pp.~303--314,
  2018.

\bibitem{samek2019}
W.~Samek, G.~Montavon, A.~Vedaldi, L.~K. Hansen, and K.-R. M{\"u}ller, {\em
  Explainable AI: interpreting, explaining and visualizing deep learning},
  vol.~11700.
\newblock Springer Nature, 2019.

\bibitem{baldassarre2019}
F.~Baldassarre and H.~Azizpour, ``Explainability techniques for graph
  convolutional networks,'' {\em arXiv preprint arXiv:1905.13686}, 2019.

\bibitem{levi2019}
D.~Levi, L.~Gispan, N.~Giladi, and E.~Fetaya, ``Evaluating and calibrating
  uncertainty prediction in regression tasks,'' {\em arXiv preprint
  arXiv:1905.11659}, 2019.

\bibitem{yao2019}
J.~Yao, W.~Pan, S.~Ghosh, and F.~Doshi-Velez, ``Quality of uncertainty
  quantification for {B}ayesian neural network inference,'' {\em arXiv preprint
  arXiv:1906.09686}, 2019.

\bibitem{leibig2017}
C.~Leibig, V.~Allken, M.~S. Ayhan, P.~Berens, and S.~Wahl, ``Leveraging
  uncertainty information from deep neural networks for disease detection,''
  {\em Scientific {R}eports}, vol.~7, no.~1, pp.~1--14, 2017.

\bibitem{michelmore2020}
R.~Michelmore, M.~Wicker, L.~Laurenti, L.~Cardelli, Y.~Gal, and M.~Kwiatkowska,
  ``Uncertainty quantification with statistical guarantees in end-to-end
  autonomous driving control,'' in {\em 2020 IEEE International Conference on
  Robotics and Automation (ICRA)}, pp.~7344--7350, IEEE, 2020.

\bibitem{abdar2020}
M.~Abdar, F.~Pourpanah, S.~Hussain, D.~Rezazadegan, L.~Liu, M.~Ghavamzadeh,
  P.~Fieguth, A.~Khosravi, U.~R. Acharya, V.~Makarenkov, {\em et~al.}, ``A
  review of uncertainty quantification in deep learning: Techniques,
  applications and challenges,'' {\em arXiv preprint arXiv:2011.06225}, 2020.

\bibitem{jospin2020}
L.~V. Jospin, W.~Buntine, F.~Boussaid, H.~Laga, and M.~Bennamoun, ``Hands-on
  {B}ayesian neural networks--a tutorial for deep learning users,'' {\em arXiv
  preprint arXiv:2007.06823}, 2020.

\bibitem{lakshminarayanan2017}
B.~Lakshminarayanan, A.~Pritzel, and C.~Blundell, ``Simple and scalable
  predictive uncertainty estimation using deep ensembles,'' in {\em Advances in
  Neural Information Processing Systems}, vol.~30, pp.~6402--6413, Curran
  Associates, Inc., 2017.

\bibitem{caldeira2020}
J.~Caldeira and B.~Nord, ``Deeply uncertain: comparing methods of uncertainty
  quantification in deep learning algorithms,'' {\em Machine Learning: Science
  and Technology}, vol.~2, no.~1, p.~015002, 2020.

\bibitem{gustafsson2020}
F.~K. Gustafsson, M.~Danelljan, and T.~B. Schon, ``Evaluating scalable
  {B}ayesian deep learning methods for robust computer vision,'' in {\em
  Proceedings of the IEEE/CVF Conference on Computer Vision and Pattern
  Recognition Workshops}, pp.~318--319, 2020.

\bibitem{ovadia2019}
Y.~Ovadia, E.~Fertig, J.~Ren, Z.~Nado, D.~Sculley, S.~Nowozin, J.~V. Dillon,
  B.~Lakshminarayanan, and J.~Snoek, ``Can you trust your model's uncertainty?
  {E}valuating predictive uncertainty under dataset shift,'' {\em arXiv
  preprint arXiv:1906.02530}, 2019.

\bibitem{scalia2020}
G.~Scalia, C.~A. Grambow, B.~Pernici, Y.-P. Li, and W.~H. Green, ``Evaluating
  scalable uncertainty estimation methods for deep learning-based molecular
  property prediction,'' {\em Journal of {C}hemical {I}nformation and
  {M}odeling}, vol.~60, no.~6, pp.~2697--2717, 2020.

\bibitem{hoffmann2021}
L.~Hoffmann, I.~Fortmeier, and C.~Elster, ``Uncertainty quantification by
  ensemble learning for computational optical form measurements,'' {\em Machine
  Learning: Science and Technology}, 2021.

\bibitem{hansen1990}
L.~K. Hansen and P.~Salamon, ``Neural network ensembles,'' {\em IEEE
  transactions on pattern analysis and machine intelligence}, vol.~12, no.~10,
  pp.~993--1001, 1990.

\bibitem{carney1999}
J.~G. Carney, P.~Cunningham, and U.~Bhagwan, ``Confidence and prediction
  intervals for neural network ensembles,'' in {\em IJCNN'99. International
  Joint Conference on Neural Networks. Proceedings (Cat. No. 99CH36339)},
  vol.~2, pp.~1215--1218, IEEE, 1999.

\bibitem{sagi2018}
O.~Sagi and L.~Rokach, ``Ensemble learning: A survey,'' {\em Wiley
  Interdisciplinary Reviews: Data Mining and Knowledge Discovery}, vol.~8,
  no.~4, p.~e1249, 2018.

\bibitem{ashukha2020}
A.~Ashukha, A.~Lyzhov, D.~Molchanov, and D.~Vetrov, ``Pitfalls of in-domain
  uncertainty estimation and ensembling in deep learning,'' {\em arXiv preprint
  arXiv:2002.06470}, 2020.

\bibitem{hu2020}
S.~Hu, N.~Pezzotti, and M.~Welling, ``A new perspective on uncertainty
  quantification of deep ensembles,'' {\em arXiv e-prints}, pp.~arXiv--2002,
  2020.

\bibitem{duvenaud2016}
D.~Duvenaud, D.~Maclaurin, and R.~Adams, ``Early stopping as nonparametric
  variational inference,'' in {\em Artificial Intelligence and Statistics},
  pp.~1070--1077, PMLR, 2016.

\bibitem{pearce2020}
T.~Pearce, F.~Leibfried, and A.~Brintrup, ``Uncertainty in neural networks:
  Approximately {B}ayesian ensembling,'' in {\em International {C}onference on
  {A}rtificial {I}ntelligence and {S}tatistics}, pp.~234--244, PMLR, 2020.

\bibitem{he2020}
B.~He, B.~Lakshminarayanan, and Y.~W. Teh, ``Bayesian deep ensembles via the
  neural tangent kernel,'' {\em arXiv preprint arXiv:2007.05864}, 2020.

\bibitem{wilson2020}
A.~G. Wilson and P.~Izmailov, ``Bayesian deep learning and a probabilistic
  perspective of generalization,'' {\em arXiv preprint arXiv:2002.08791}, 2020.

\bibitem{hoeting1999}
J.~A. Hoeting, D.~Madigan, A.~E. Raftery, and C.~T. Volinsky, ``Bayesian model
  averaging: A tutorial,'' {\em Statistical Science}, vol.~14, no.~4,
  pp.~382--401, 1999.

\bibitem{ritter2018}
H.~Ritter, A.~Botev, and D.~Barber, ``A scalable {L}aplace approximation for
  neural networks,'' in {\em 6th International Conference on Learning
  Representations, ICLR 2018-Conference Track Proceedings}, vol.~6,
  International Conference on Representation Learning, 2018.

\bibitem{maddox2019}
W.~J. Maddox, P.~Izmailov, T.~Garipov, D.~P. Vetrov, and A.~G. Wilson, ``A
  simple baseline for {B}ayesian uncertainty in deep learning,'' {\em Advances
  in Neural Information Processing Systems}, vol.~32, pp.~13153--13164, 2019.

\bibitem{lampinen2001}
J.~Lampinen and A.~Vehtari, ``Bayesian approach for neural networks—review
  and case studies,'' {\em Neural networks}, vol.~14, no.~3, pp.~257--274,
  2001.

\bibitem{blei2017}
D.~M. Blei, A.~Kucukelbir, and J.~D. McAuliffe, ``Variational inference: A
  review for statisticians,'' {\em Journal of the American Statistical
  Association}, vol.~112, no.~518, pp.~859--877, 2017.

\bibitem{zeng2018}
J.~Zeng, A.~Lesnikowski, and J.~M. Alvarez, ``The relevance of {B}ayesian layer
  positioning to model uncertainty in deep {B}ayesian active learning,'' {\em
  arXiv preprint arXiv:1811.12535}, 2018.

\bibitem{brosse2020}
N.~Brosse, C.~Riquelme, A.~Martin, S.~Gelly, and {\'E}.~Moulines, ``On
  last-layer algorithms for classification: Decoupling representation from
  uncertainty estimation,'' {\em arXiv preprint arXiv:2001.08049}, 2020.

\bibitem{kristiadi2020}
A.~Kristiadi, M.~Hein, and P.~Hennig, ``Being {B}ayesian, even just a bit,
  fixes overconfidence in relu networks,'' in {\em International Conference on
  Machine Learning}, pp.~5436--5446, PMLR, 2020.

\bibitem{gelman2013}
A.~Gelman, J.~B. Carlin, H.~S. Stern, D.~B. Dunson, A.~Vehtari, and D.~B.
  Rubin, {\em Bayesian data analysis}.
\newblock CRC {P}ress, 2013.

\bibitem{robert2007}
C.~Robert, {\em The {B}ayesian choice: from decision-theoretic foundations to
  computational implementation}.
\newblock Springer Science \& Business Media, 2007.

\bibitem{ohagan2004}
A.~O'Hagan and J.~J. Forster, {\em Kendall's advanced theory of statistics,
  volume 2B: {B}ayesian inference}, vol.~2.
\newblock Arnold, 2004.

\bibitem{gal2016thesis}
Y.~Gal, {\em Uncertainty in Deep Learning}.
\newblock PhD thesis, University of Cambridge, 2016.

\bibitem{hullermeier2021}
E.~H{\"u}llermeier and W.~Waegeman, ``Aleatoric and epistemic uncertainty in
  machine learning: An introduction to concepts and methods,'' {\em Machine
  Learning}, vol.~110, no.~3, pp.~457--506, 2021.

\bibitem{fort2019}
S.~Fort, H.~Hu, and B.~Lakshminarayanan, ``Deep ensembles: A loss landscape
  perspective,'' {\em arXiv preprint arXiv:1912.02757}, 2019.

\bibitem{dua2019}
D.~Dua and C.~Graff, ``{UCI} machine learning repository,'' 2017.

\bibitem{kingma2014adam}
D.~P. Kingma and J.~Ba, ``Adam: A method for stochastic optimization,'' {\em
  arXiv preprint arXiv:1412.6980}, 2014.

\bibitem{cassotti2014prediction}
M.~Cassotti, D.~Ballabio, V.~Consonni, A.~Mauri, I.~V. Tetko, and
  R.~Todeschini, ``Prediction of acute aquatic toxicity toward daphnia magna by
  using the ga-k nn method,'' {\em Alternatives to Laboratory Animals},
  vol.~42, no.~1, pp.~31--41, 2014.

\bibitem{gerritsma1981geometry}
J.~Gerritsma, R.~Onnink, and A.~Versluis, ``Geometry, resistance and stability
  of the delft systematic yacht hull series,'' {\em International shipbuilding
  progress}, vol.~28, no.~328, pp.~276--297, 1981.

\bibitem{ortigosa2007neural}
I.~Ortigosa, R.~Lopez, and J.~Garcia, ``A neural networks approach to residuary
  resistance of sailing yachts prediction,'' in {\em Proceedings of the
  international conference on marine engineering MARINE}, vol.~2007, p.~250,
  2007.

\bibitem{buza2014feedback}
K.~Buza, ``Feedback prediction for blogs,'' in {\em Data analysis, machine
  learning and knowledge discovery}, pp.~145--152, Springer, 2014.

\bibitem{bertin-mahieux2011}
T.~Bertin-Mahieux, D.~P. Ellis, B.~Whitman, and P.~Lamere, ``The million song
  dataset,'' in {\em {Proceedings of the 12th International Conference on Music
  Information Retrieval ({ISMIR} 2011)}}, 2011.

\bibitem{wenzel2020}
F.~Wenzel, J.~Snoek, D.~Tran, and R.~Jenatton, ``Hyperparameter ensembles for
  robustness and uncertainty quantification,'' {\em arXiv preprint
  arXiv:2006.13570}, 2020.

\end{thebibliography}

\appendixtitleon
\appendixtitletocon
\begin{appendices}
\section{ }\label{appendix theory}
\subsection{Mean and covariance matrix of the posterior predictive distribution for the original deep ensembles approach}\label{appendix a1}

The first \eqref{eq1.2} and central second \eqref{eq1.3} moments of the approximation \eqref{eq1.1b} of the posterior predictive distribution are calculated as follows:

\begin{eqnarray*}
E \left( y |x,D \right) &=& \int y \pi(y|x,D) \dd y \\
&\approx& \frac{1}{L} \sum_{l=1}^L \int y N \left( y; \eta_{\widehat{\theta}^{(l)}}(x), \sigma_{\widehat{\theta}^{(l)}}^2(x) \bI \right) \dd y  \\
&=& \frac{1}{L} \sum_{l=1}^L  \eta_{\widehat{\theta}^{(l)}}(x) \,, \\
Cov \left( y |x,D \right) 
&=&\int  \left[ \left(y-E \left( y |x,D \right)\right) \left(y-E \left( y |x,D \right)\right)^T \right] \pi(y|x,D) \dd y  \\
&\approx& \frac{1}{L} \sum_{l=1}^L \int [ \left(\eta_{\widehat{\theta}^{(l)}}(x)-E \left( y |x,D \right)\right) \left(\eta_{\widehat{\theta}^{(l)}}(x)-E \left( y |x,D \right)\right)^T \\
&+& \left(y-\eta_{\widehat{\theta}^{(l)}}(x)\right)\left(\eta_{\widehat{\theta}^{(l)}}(x)-E \left( y |x,D \right)\right)^T \\
&+& \left(\eta_{\widehat{\theta}^{(l)}}(x)-E \left( y |x,D \right)\right) \left(y-\eta_{\widehat{\theta}^{(l)}}(x)\right)^T\\
&+& \left(y-\eta_{\widehat{\theta}^{(l)}}(x)\right) \left(y-\eta_{\widehat{\theta}^{(l)}}(x)\right)^T ]\ N \left( y; \eta_{\widehat{\theta}^{(l)}}(x), \sigma_{\widehat{\theta}^{(l)}}^2(x) \bI \right) \dd y  \\
&=& \frac{1}{L} \sum_{l=1}^L  \left[\left(\eta_{\widehat{\theta}^{(l)}}(x)-E \left( y |x,D \right)\right) \left(\eta_{\widehat{\theta}^{(l)}}(x)-E \left( y |x,D \right)\right)^T + \sigma_{\widehat{\theta}^{(l)}}^2(x) \bI\right] \,.
\end{eqnarray*}

\subsection{Proof of Lemma \ref{T2}}\label{appendix T2}
Application of the change-of-variables formula and the approximation \eqref{eq1.1} of the posterior $\pi(\theta|D)$ yield the following for the regression function $\eta\equiv\eta_\theta (x)$:
\begin{eqnarray*}
\pi(\eta |x,D) &=& \int \delta(\eta-\eta_\theta(x)) \pi(\theta|D) \dd \theta \\
&\approx& \int \delta(\eta-\eta_\theta(x)) \frac{1}{L} \sum_{l=1}^L \delta(\theta - \whtl)  \dd \theta  \\
&=& \frac{1}{L} \sum_{l=1}^L \delta(\eta- \eta_{\widehat{\theta}^{(l)}}(x))\,.
\end{eqnarray*}
The expectation and covariance matrix then follow immediately:

\begin{eqnarray*}
E \left( \eta |x,D \right) &=& \int \eta \pi(\eta|x,D) \dd \eta \\
&\approx& \frac{1}{L} \sum_{l=1}^L \int \eta \delta(\eta- \eta_{\widehat{\theta}^{(l)}}(x))  \dd \eta \\
&=& \frac{1}{L} \sum_{l=1}^L  \eta_{\widehat{\theta}^{(l)}}(x) \,,
\end{eqnarray*}
\begin{eqnarray*}
Cov \left( \eta |x,D \right) 
&=&\int  \left[ \left(\eta-E \left( \eta |x,D \right)\right) \left(\eta-E \left( \eta |x,D \right)\right)^T \right] \pi(\eta|x,D) \dd \eta  \\
&\approx& \frac{1}{L} \sum_{l=1}^L \int [ \left(\eta_{\widehat{\theta}^{(l)}}(x)-E \left( \eta |x,D \right)\right) \left(\eta_{\widehat{\theta}^{(l)}}(x)-E \left( \eta |x,D \right)\right)^T \\
&+& \left(\eta-\eta_{\widehat{\theta}^{(l)}}(x)\right)\left(\eta_{\widehat{\theta}^{(l)}}(x)-E \left( \eta |x,D \right)\right)^T \\
&+& \left(\eta_{\widehat{\theta}^{(l)}}(x)-E \left( \eta |x,D \right)\right) \left(\eta-\eta_{\widehat{\theta}^{(l)}}(x)\right)^T\\
&+& \left(\eta-\eta_{\widehat{\theta}^{(l)}}(x)\right) \left(\eta-\eta_{\widehat{\theta}^{(l)}}(x)\right)^T ]\ \delta(\eta- \eta_{\widehat{\theta}^{(l)}}(x)) \dd \eta  \\
&=& \frac{1}{L} \sum_{l=1}^L  \left[\left(\eta_{\widehat{\theta}^{(l)}}(x)-E \left( \eta |x,D \right)\right) \left(\eta_{\widehat{\theta}^{(l)}}(x)-E \left( \eta |x,D \right)\right)^T\right] \,. \qed
\end{eqnarray*}

\subsection{Proof of Theorem \ref{T3}}\label{appendix T3}
The objective function to be maximized is the ELBO from equation \eqref{eq3.1}:
\begin{equation*}
ELBO = E_q \log p(D|\theta) - KL \left( q(\theta) \| \pi(\theta) \right)\, .
\end{equation*}

The first term of equation \eqref{eq3.1} is the expectation of the log likelihood (up to a constant):
\begin{equation}
E_q \log p(D|\theta) = - \frac{1}{2L} \sum_{l=1}^L \sum_{i=1}^N  \int N \left( \theta; \whtl, \gamma_l \bI \right) \left(
\frac{ \| y_i-\eta_\theta(x_i) \|^2}{\swhtl(x_i)} + p_y \log(\swhtl(x_i)) \right) \dd\theta, \label{eq7.4}
\end{equation}

where $\Vert . \Vert$ refers to the $L^2$ norm. Note that the parameters of the last layer related to the variances $\swhtl$ of the statistical model \eqref{eq1.0} are fixed after the network training. The parameters $\theta$ of the corresponding mean $\eta_\theta$ are drawn only from the last layer to the output without any nonlinearity. Hence, the network output of the $l$th ensemble member can  be written as
\begin{equation}
\eta_\theta(x) = W_\theta\ \whel(x),\ W_\theta=W_{\hat{\theta}^{(l)}}+\sqrt{\gamma_l}\mathcal{E}\ \in\mathbb{R}^{p_y\times p_{\hat{\eta}}},\ \mathcal{E} .\sim N(0,I),\label{eq7.1}
\end{equation}
where $\whel(x)$ is the output of the next-to-last layer of the trained network. The dependence of $\whel(x)$ on $\theta$ is suppressed, assuming that the corresponding MAP estimates for $\theta$ are taken. $W_\theta$ contains the parameters of the linear unit drawn from a normal distribution around the MAPs $W_{\hat{\theta}^{(l)}}$ with the same elementwise variance $\gamma_l$ for each parameter. Here, the operation ``$.\sim$" means ``elementwise drawn". $\mathcal{E}$ is a matrix with the same dimensionality as $W_{\hat{\theta}^{(l)}}$, i.e. the output dimension $p_y$ times the dimension of the last layer $p_{\hat{\eta}}$.

Using \eqref{eq7.1}, the following equations hold:
\begin{flalign}
&\int N \left( \theta; \whtl, \gamma_l \bI \right)  \| y-\eta_\theta(x) \|^2 \dd\theta \nonumber\\
&= E_\mathcal{E}\left[ \Vert y - (W_{\hat{\theta}^{(l)}} + \sqrt{\gamma_l}\mathcal{E})\hat{\eta}^{(l)}(x)\Vert^2 \right] \\
&= E_\mathcal{E}\left[ 
\Vert y - W_{\hat{\theta}^{(l)}}\hat{\eta}^{(l)}(x)\Vert^2
- 2 (y - W_{\hat{\theta}^{(l)}} \hat{\eta}^{(l)}(x))^T ( \sqrt{\gamma_l}\mathcal{E}\hat{\eta}^{(l)}(x)) 
+ \Vert \sqrt{\gamma_l}\mathcal{E}\hat{\eta}^{(l)}(x)\Vert^2 
\right] \nonumber \\
&= \Vert y -  \eta_{\hat{\theta}^{(l)}}(x)\Vert^2
+ \gamma_l E_\mathcal{E}\left[ \Vert \mathcal{E}\hat{\eta}^{(l)}(x)\Vert^2 \right]\,,\label{eq7.2}
\end{flalign}
where $\eta_{\hat{\theta}^{(l)}}$ is the trained network with MAP estimates. Furthermore, the following holds:
\begin{eqnarray}
E_\mathcal{E}\left[ \Vert \mathcal{E}\whel(x)\Vert^2 \right]
&=& E_\mathcal{E}\left[ (\mathcal{E}\whel(x))^T(\mathcal{E}\whel(x)) \right] \\
&=& \sum_{r=1}^{p_y}\sum_{\alpha,\beta=1}^{p_{\hat{\eta}}} \whel(x)_\alpha E\left[\mathcal{E}_{r,\alpha} \mathcal{E}_{r,\beta} \right] \whel(x)_\beta  \\
&=& \sum_{r=1}^{p_y}\sum_{\alpha=1}^{p_{\hat{\eta}}} \whel(x)_\alpha E\left[\mathcal{E}_{r,\alpha}^2 \right] \whel(x)_\alpha  \\
&=& \sum_{r=1}^{p_y}\sum_{\alpha=1}^{p_{\hat{\eta}}} \whel(x)_\alpha \left( \Var(\mathcal{E}_{r,\alpha}) + E\left[\mathcal{E}_{r,\alpha}\right]^2 \right) \whel(x)_\alpha  \\
&=& p_y \Vert\whel(x)\Vert^2\,. \label{eq7.3}
\end{eqnarray}
Inserting \eqref{eq7.2} and \eqref{eq7.3} into equation \eqref{eq7.4} yields:
\mathleft
\begin{equation}
E_q \log p(D|\theta) = - \frac{1}{2L} \sum_{l=1}^L \sum_{i=1}^N \left(
\frac{\Vert y_i - \eta_{\hat{\theta}^{(l)}}(x)\Vert^2 + \gamma_l p_y \Vert\whel(x)\Vert^2 }{\swhtl(x_i)} + p_y \log(\swhtl(x_i)) \right) . \label{eq7.5}
\end{equation}

The second term of the ELBO \eqref{eq3.1} is the Kullback-Leibler divergence between the approximation of the posterior $q(\theta)$ and the chosen prior $\pi (\theta)$. The following holds:
\mathleft
\begin{eqnarray}
KL \left( q(\theta) \| \pi(\theta) \right)
&=&
\int q(\theta) \log(\frac{q(\theta)}{\pi(\theta)}) \dd \theta \\
&=&
\int \frac{1}{L} \sum_{l=1}^L N \left( \theta; \whtl, \gamma_l \bI \right)
\log \left( \frac{\frac{1}{L} \sum_{r=1}^L N \left( \theta; \whtr, \gamma_r \bI \right)}{\pi(\theta)} \right) \dd \theta \\
&\approx&
\frac{1}{L} \sum_{l=1}^L \int 
N \left( \theta; \whtl, \gamma_l \bI \right)
\log \left( \frac{\frac{1}{L} N \left( \theta; \whtl, \gamma_l \bI \right) }{\pi(\theta) } \right) \dd \theta \\
&=&
\frac{1}{L} \sum_{l=1}^L KL \left(N ( \theta; \whtl, \gamma_l \bI ) \| \pi(\theta) \right) - \frac{1}{L} \sum_{l=1}^L\log(L)\,. \label{eq7.6}
\end{eqnarray}
The approximation
\begin{eqnarray}
&\int  N \left( \theta; \whtl, \gamma_l \bI \right)
\log \left( \sum_{r=1}^L N \left( \theta; \whtr, \gamma_r \bI \right) \right) \dd \theta \nonumber \\
&\approx
\int  N \left( \theta; \whtl, \gamma_l \bI \right)
\log \left( N \left( \theta; \whtl, \gamma_l \bI \right) \right) \dd \theta
\end{eqnarray}
is valid under the assumption that the between-variability of the MAP estimates is large compared to the within-variability, since 
then the first term of the integrand (the normal distribution around $\whtl$) goes sufficiently fast to zero  for distant $\whtr, r\neq l$.

It is well known that the Kullback-Leibler divergence between two Gaussians can be calculated analytically. It follows that:
\begin{eqnarray}
KL \left(N ( \theta; \whtl, \gamma_l \bI ) \| \pi(\theta) \right) &=&
KL \left(N ( \theta; \whtl, \gamma_l \bI ) \| N (\theta; 0,\lambda^{-1}I) \right)  \\
&=& \frac{1}{2} \left( 
p_\theta \gamma_l \lambda + \lambda \Vert\whtl \Vert^2 - p_\theta
- p_\theta \log(\gamma_l \lambda) 
\right),
\end{eqnarray}
where $p_\theta:=dim(\whtl),\ l=1,\ldots,L$. Here, $p_\theta = p_y p_{\hat{\eta}}$, as only the weights $W_\theta$ of the last linear unit of the network are random. It follows with \eqref{eq7.6} that
\begin{equation}
KL \left( q(\theta) \| \pi(\theta) \right)
= \frac{1}{2L} \sum_{l=1}^L
\left( 
p_y p_{\hat{\eta}} \gamma_l \lambda + \lambda \Vert W_{\hat{\theta}^{(l)}} \Vert^2_F - p_y p_{\hat{\eta}}
- p_y p_{\hat{\eta}} \log(\gamma_l \lambda) 
\right)
- \frac{1}{L} \sum_{l=1}^L\log(L)\,, \label{eq7.7}
\end{equation}
where $\| \cdots \|_{F}$ denotes the Frobenius norm.

Now, the ELBO is given by combining \eqref{eq7.5} and \eqref{eq7.7}:
\begin{eqnarray}
ELBO &=& E_q \log p(D|\theta) - KL \left( q(\theta) \| \pi(\theta) \right) \nonumber\\
&=& - \frac{1}{2L} \sum_{l=1}^L \sum_{i=1}^N \left(
\frac{\Vert y_i - \eta_{\hat{\theta}^{(l)}}(x_i)\Vert^2 + \gamma_l p_y \Vert\whel(x_i)\Vert^2 }{\swhtl(x_i)} + p_y \log(\swhtl(x_i)) \right) \nonumber \\
&-& \left( \frac{1}{2L} \sum_{l=1}^L
\left( 
p_y p_{\hat{\eta}} \gamma_l \lambda + \lambda \Vert W_{\hat{\theta}^{(l)}} \Vert^2_F - p_y p_{\hat{\eta}}
- p_y p_{\hat{\eta}} \log(\gamma_l \lambda) 
\right)
- \frac{1}{L} \sum_{l=1}^L\log(L) \right) \nonumber
\end{eqnarray}
\mathleft
\begin{eqnarray}
&=& \frac{1}{L} \sum_{l=1}^L [
\log(L)
- \frac{1}{2} \{ 
\sum_{i=1}^N \left(
\frac{\Vert y_i - \eta_{\hat{\theta}^{(l)}}(x_i)\Vert^2 }{\swhtl(x_i)} + p_y \log(\swhtl(x_i)) \right) 
+
\lambda \Vert W_{\hat{\theta}^{(l)}} \Vert^2_F - p_y p_{\hat{\eta}} \nonumber \\
&-& p_y p_{\hat{\eta}} \log( \lambda) \} - \frac{1}{2} \left(\sum_{i=1}^N \frac{p_y \Vert\whel(x_i)\Vert^2 }{\swhtl(x_i)} + p_y p_{\hat{\eta}} \lambda \right)\gamma_l
+ \frac{1}{2}p_y p_{\hat{\eta}} \log(\gamma_l)
].
\end{eqnarray}
In total, the ELBO takes the form:
\begin{equation}
ELBO = \frac{1}{L} \sum_{l=1}^L 
\left( a_l + b_l \gamma_l + c_l \log(\gamma_l) \right)\,,\label{eq7.8}
\end{equation}
where $a_l, b_l, c_l$ are given analytically in dependence on the (local) MAP estimates $\whtl, l=1, \ldots,L$.
The maximizer of \eqref{eq7.8} w.r.t. the $\gamma_{l}$ is given by
\begin{equation}\label{eq7.9}
\gamma_l = - \frac{c_l}{b_l}
= \frac{ p_{\hat{\eta}}}{\sum_{i=1}^N \frac{\Vert\whel(x_i)\Vert^2 }{\swhtl(x_i)} + p_{\hat{\eta}} \lambda}
\,,\ l=1,\ldots,L\,.      
\end{equation}
Note that $b_l \ne 0$ holds.
It can be expected that $\sum_{i=1}^N \frac{\Vert\whel(x_i)\Vert^2 }{\swhtl(x_i)}$ is unbounded as $N$ grows, since the individual networks are trained independently. Under this assumption,
it follows that $\gamma_l\to 0$ for $N\to\infty$.
$\qed$

\subsection{Proof of Theorem \ref{T4}}\label{appendix T4}
The proof is straightforward in analogy to the proof of Theorem \ref{T2} (cf. Appendix \ref{appendix T2}), with the following addition:
\begin{equation}
    \int \left(\eta-\eta_{\widehat{\theta}^{(l)}}(x)\right) \left(\eta-\eta_{\widehat{\theta}^{(l)}}(x)\right)^T\ \delta(\eta- \eta_\theta(x)) N \left( \theta; \whtl, \gamma_l \bI \right) \dd \theta \dd \eta \approx \gamma_l J_l J_l^T\, ,\ l=1,\ldots,L\,,
\end{equation}
using the first-order Taylor approximation $\eta_\theta (x) \approx \eta_{\widehat{\theta}^{(l)}}(x) + J_l (\theta - \whtl)$
of $\eta_\theta$ for $\theta$ in the vicinity of $\widehat{\theta}^{(l)}$.

The derivative $J_l$ can be easily calculated based on the output of the next-to-last layer of the trained networks (cf. Appendix \ref{appendix T3} \eqref{eq7.1}) and equation \eqref{eq6.1}.
\qed

\end{appendices}
\end{document}